\documentclass{article}
\usepackage{arxiv}

\usepackage{amsmath,amssymb,amsthm}
\usepackage{booktabs}
\usepackage{algorithmic}
\usepackage[ruled, noend]{algorithm2e}
\usepackage{changepage}
\usepackage{amsfonts}
\usepackage{adjustbox}
\usepackage{balance}
\usepackage{multirow, rotating}
\usepackage{comment}
\usepackage{fancyhdr}
\usepackage{graphicx}
\usepackage{float} 
\usepackage{amsmath} 
\usepackage{subfigure}
\usepackage{dsfont}
\newcommand{\Var}{\mathrm{Var}\xspace}

\newcommand{\ignore}[1]{}
\usepackage{color}

\newtheorem{theorem}{Theorem}
\newtheorem{lemma}{Lemma}

\begin{document}

\author{Aneta Neumann
\\Optimisation and Logistics\\
School of Computer Science
\\The University of Adelaide
\\Adelaide, Australia
\And
Yue Xie
\\Optimisation and Logistics\\
School of Computer Science
\\The University of Adelaide
\\Adelaide, Australia
\And
Frank Neumann
\\Optimisation and Logistics\\
School of Computer Science
\\The University of Adelaide
\\Adelaide, Australia
}

\title{Evolutionary Algorithms for Limiting the Effect of Uncertainty for the Knapsack Problem with Stochastic Profits}

\maketitle

\begin{abstract}
Evolutionary algorithms have been widely used for a range of stochastic optimization problems in order to address complex real-world optimization problems. We consider the knapsack problem where the profits involve uncertainties. Such a stochastic setting reflects important real-world scenarios where the profit that can be realized is uncertain. We introduce different ways of dealing with stochastic profits based on tail inequalities such as Chebyshev's inequality and Hoeffding bounds that allow to limit the impact of uncertainties. 
We examine simple evolutionary algorithms and the use of heavy tail mutation and a problem-specific crossover operator for optimizing uncertain profits. Our experimental investigations on different benchmarks instances show the results of different approaches based on tail inequalities as well as improvements achievable through heavy tail mutation and the problem specific crossover operator.
\end{abstract}


\section{Introduction}

Evolutionary algorithms~\cite{DBLP:series/ncs/EibenS15} have been successfully applied to a wide range of complex optimization problems~\cite{DBLP:books/sp/chiong12,DBLP:conf/gecco/MyburghD10,DBLP:conf/cec/OsadaWBM13}. 
Stochastic problems play a crucial role in the area of optimization and evolutionary algorithms have frequently been applied to noisy environments~\cite{DBLP:journals/swevo/RakshitKD17}.

Given a stochastic function to be optimized under a given set of constraints, the goal is often to maximize the expected value of a solution with respect to $f$.
This however does not consider the deviation from the expected value. Guaranteeing that a function value with a good probability does not drop below a certain value is often more beneficial in real-world scenarios. For example, in the area of mine planning~\cite{marcotte2013ultimate,capponi2020mine}, profits achieved within different years should be maximized. However, it is crucial to not drop below certain profit values because then the whole mine operation would not be viable and the company might go bankrupt.

We consider a stochastic version of the knapsack problem which fits the characteristics of the mine planning problem outline above. Here the profits are stochastic and the weights are deterministic.
Motivated by the area of chance constrained optimization where constraints can only be violated with a small probability, we consider the scenario where we maximize the function value $P$ for which we can guarantee that the best solution $x$ as a profit less than $P$ with probability at most $\alpha_p$, i.e. $Prob (p(x) < P) \leq \alpha_p$. 
Note that determining whether $Prob (p(x) < P) \leq \alpha_p$ holds for a given solution $x$ and values $P$ and $\alpha_p$ is already hard for very simple stochastic settings where profits are independent and each profit can only take on two different values.
Furthermore, finding a solution with a maximal $P$ for which the condition holds poses in general a non-linear objective function that needs to take the probability distribution of $p(x)$ into account.
Constraints of the beforehand mentioned type are known as chance constraints~\cite{charnes1959chance}. Chance constraints on stochastic components of a problem can only be violated with a small probability, in our case specified by the parameter $\alpha_p$.

\subsection{Related work}
Up to recently, only a few problems with chance constraints have been studied in the evolutionary computation literature~\cite{10.5555/2933923.2933961,doi:10.1080/03057920412331272144,DBLP:journals/tec/LiuZFG13}. They are based on simulations and sampling techniques for evaluating chance constraints. Such approaches require a relatively large computation time for evaluating the chance constraints. In contrast to this, tail inequalities can be used if certain characteristics such as the expected value and variance of a distribution are known. Such approaches have recently been examined for the chance constrained knapsack problem in the context of evolutionary computation~\cite{DBLP:conf/gecco/XieHAN019,DBLP:conf/gecco/XieN020}.
The standard version of the chance-constrained knapsack problem considers the case where the profits are deterministic and the weights are stochastic. Here the constraint bound $B$ can only be violated with a small probability. Different single and multi-objective approaches have recently been investigated~\cite{DBLP:conf/gecco/XieHAN019,DBLP:conf/gecco/XieN020} and also the case of a dynamically changing constraint has been investigated~\cite{DBLP:journals/corr/abs-2002-06766}

Furthermore, chance constrained monotone submodular functions have been studied in~\cite{DBLP:journals/corr/abs-1911-11451,DBLP:conf/ppsn/NeumannN20}. In~\cite{DBLP:journals/corr/abs-1911-11451}, greedy algorithms that use tail inequalities such as Chebyshev's inequality and Chernoff bounds have been analyzed. It has been shown that they almost achieve the same approximation guarantees in the chance constrained setting as in the deterministic setting. In~\cite{DBLP:conf/ppsn/NeumannN20}, the use of evolutionary multi-objective algorithms for monotone submodular functions has been investigated and it has been shown that they achieve the same approximation guarantees as the greedy algorithms but perform much better in practice.

Finally, chance constrained problems have been further investigated through runtime analysis for special instances of the knapsack problem~\cite{DBLP:conf/foga/0001S19,DBLP:journals/corr/abs-2102-05778}. This includes a first study on very specific instances showing when local optima arise~\cite{DBLP:conf/foga/0001S19} and a study on groups of items whose stochastic uniform weights are correlated with each other~\cite{DBLP:journals/corr/abs-2102-05778}.

All previously mentioned studies concentrated on stochastic weights and how algorithms can deal with the chance constraints with respect to the weight bound of the knapsack problem. In~\cite{Morton1998}, a version of the knapsack problem stochastic profits and deterministic weights has been considered where the goal is to maximize the probability that the profit meets a certain threshold value. In contrast to this, we will maximize the profit under the condition that it is achieved with high probability.
We will provide a first study on evolutionary algorithms for giving guarantees when maximizing stochastic profits, a topic that is well motivated by the beforehand mentioned mine planning application but to our knowledge not studied in the literature yet.

The paper is structured as follows. In Section~\ref{sec2}, we introduce the problem formulation and tail bounds that will be used to construct fitness functions for dealing with stochastic profit. In Section~\ref{sec3}, we derive fitness functions that are able to maximize the profit for which we can give guarantees. Section~\ref{sec4} introduces evolutionary algorithms for the problem and we report on our experimental investigations in Section~\ref{sec5}. We finally finish with some conclusions.

\section{Problem definition}
\label{sec2}
In this section, we formally introduce the problem and tail inequalities for dealing with stochastic profits that will later be used to design fitness functions.
We consider a stochastic version of the classical NP-hard knapsack problem. In the classical problem, there are given $n$ items $1, \ldots, n$ where each item has a profit $p_i$ and a weight $w_i$, the goal is to maximize the profit $p(x)=\sum_{i=1}^n p_i$ under the condition that $w(x) = \sum_{i=1}^n w_i x_i \leq B$ for a given weight bound $B$ holds. The classical knapsack problem has been well studied in the literature. 
We consider the following stochastic version, where the profits $p_i$ are stochastic and the weights are still deterministic. Our goal is to maximize the profit $P$ for which we can guarantee that there is only a small probability $\alpha_p$ of dropping below $P$.
Formally, we tackle the following problem.

\begin{eqnarray}
\label{eq:maxP}
   & \max P  \\
s. t.   & Pr(p(x) < P) \leq \alpha_p \label{eq:chance}\\
  & w(x) \leq B \label{eq:weight}\\
 & x \in \{0,1\}^n
\end{eqnarray}

Equation~\ref{eq:chance} is a chance constraint on the profit and the main goal of this paper is to find a solution $x$ that maximize the value of $P$ such that the probability of getting a profit lower than $P$ is at most $\alpha_p$.

We denote by $\mu(x)$ the expected profit and by $v(x)$ the variance of the profit throughout this paper.

\subsection{Concentration bounds}
In order to establish guarantees for the stochastic knapsack problem we make use of well-known tail inequalities that limit the deviation from the expected profit of a solution.

For a solution $X$ weight expected value $E[X]$ and variance $Var(X)$ we can use the lower tail of the following Chebyshev-Cantelli inequality.

\begin{theorem}[One-sided Chebyshev's / Cantelli's inequality]
\label{thm:geq-chebyshev}
  Let $X$ be a random variable with expected value $E[X]$ and variance $\Var[X] > 0$. Then, for all $\lambda > 0$, 
  \begin{align}
  &\Pr[X \geq E[X] - \lambda] \geq 1- \frac{\Var[X]}{\Var[X]+\lambda^2}
  \end{align}
\end{theorem}	

We will refer to this inequality as Chebyshev's inequality in the following. Chebyshev's inequality only requires the expected value and variance of a solution, but no additional requirements such as the independence of the random variables.

We use the additive Hoeffding bound given in Theorem 1.10.9 of~\cite{DoerrProbabilisticTools} for the case where the weights are independently chosen within given intervals.

\begin{theorem}[Hoeffding bound]
\label{thm:addHoeffding}
Let $X_1, \ldots, X_n$ be independent random variables. Assume that each $X_i$ takes values in a real interval $[a_i, b_i]$ of length $c_i := b_i -a_i$.
Let $X = \sum_{i=1}^n X_i$. Then for all $\lambda > 0$,
\begin{equation}
    Pr (X \geq E[x]+ \lambda) \leq e^{-2 \lambda^2/(\sum_{i=1}^n c_i^2)}
\end{equation}

\begin{equation}
    Pr (X \leq E[x]- \lambda) \leq e^{-2 \lambda^2/(\sum_{i=1}^n c_i^2)}
\end{equation}
\end{theorem}

\section{Fitness functions for profit guarantees}
\label{sec3}
The main task when dealing with the setting of chance constraint profits is to come up with fitness functions that take the uncertainty into account.

In this section, we introduce different fitness functions that can be used in an evolutionary algorithm to compute solutions that maximize the profit under the uncertainty constraint.
We consider the search space $\{0,1\}^n$ and for a given search point $x \in \{0,1\}^n$, item $i$ chosen iff $x_i=1$ holds.

The fitness of a search point $x \in \{0,1\}^n$ is given by 

$$f(x) = (u(x), \hat{p}(x))$$
where
$u(x)=\max\{w(x)-B, 0\}$ is the amount of constraint violation of the bound $B$ by the weight that should be minimized and $\hat{p}(x)$ is the discounted profit of solution $x$ that should be maximized. 
We optimize $f$ with respect to lexicographic order and have $$f(x) \geq f(y)$$ iff 
$$(u(x) < u(y)) \vee ((u(x)=u(y)) \wedge (\hat{p}(x) \geq \hat{p}(y)).$$

This implies a standard penalty approach where the weight $w(x)$ is reduced until it meets the constraint bound $B$, and the profit $\hat{p}(x)$ is maximized among all feasible solutions.

The key part if to develop formulations for $\hat{p}$ that take into account the stochastic part of the profits to make the formulations suitable for our chance constrained setting. Therefore, we will develop profit functions that reflect different stochastic settings in the following.

\subsection{Chebyshev's inequality}
We give a formulation for $\hat{p}$ that can be applied in quite general settings, thereby providing only a lower bound on the value $P$ for which a solution $x$ still meets the profit chance constraint.

We assume that for a given solution only the expected value $\mu(x)$ and the variance $v(x)$ are known. The following lemma gives a condition for meeting the chance constraint based on Theorem~\ref{thm:geq-chebyshev}.

\begin{lemma}
\label{lem:cheby}
Let $x$ be a solution with expected profit $\mu(x)$ and variance $v(x)$. If
$$
\mu(x) -P \geq \sqrt{\frac{(1-\alpha) \cdot v(x)}{\alpha}}
$$
 then $Pr(p(x) < P) \leq \alpha_p$.
\end{lemma}

\begin{proof}
We have

\begin{eqnarray*}
Pr(p(x) \geq P) & = & Pr(p(x) \geq \mu(x) - (\mu(x) -P))\\
& \geq & 1- \frac{v(x)}{v(x) + (\mu(x) -P)^2}
\end{eqnarray*}

The chance constraint is met if
\begin{eqnarray*}
& & 1- \frac{v(x)}{v(x) + (\mu(x) -P)^2} \geq 1-\alpha_p \\
 &\Longleftrightarrow  & \alpha_p \geq \frac{v(x)}{v(x) + (\mu(x) -P)^2}  \\
 &\Longleftrightarrow  & \alpha_p \cdot (v(x) + (\mu(x) -P)^2) \geq v(x)  \\
 &\Longleftrightarrow  & \alpha_p \cdot v(x) + \alpha_p \cdot (\mu(x) -P)^2) \geq v(x)  \\
  &\Longleftrightarrow  & \alpha_p(\mu(x) -P)^2) \geq (1-\alpha_p) v(x)  \\
   &\Longleftrightarrow  & (\mu(x) -P)^2) \geq ((1-\alpha_p)/\alpha_p) \cdot v(x)  \\
   &\Longrightarrow  & \mu(x) - P \geq \sqrt{(1-\alpha_p)/\alpha_p \cdot v(x)}  \\
   &\Longleftrightarrow & \mu(x)- \sqrt{((1-\alpha_p)/\alpha_p) \cdot v(x)}  \geq P
\end{eqnarray*}

\end{proof}
Given the last expression,
$P$ is maximal for 
$$P = \mu(x) - \sqrt{((1-\alpha_p)/\alpha_p) \cdot v(x)}.$$
We use the following profit function based on Chebyshev's inequality:

\begin{equation}
\hat{p}_{Cheb}(x) = \mu(x) - \sqrt{(1-\alpha_p)/\alpha_p }\cdot \sqrt{v(x)}
\end{equation}

\subsection{Hoeffding bound}
\label{sec:hoeffding}

We now assume that each element $i$ takes on a profit $p_i \in [\mu_i- \delta_p, \mu_i+ \delta_p]$ independently of the other items.
Let $\mu(x) = \sum \mu_i x_i$. we have $p(x) = \mu(x)-\delta_p |x|_1 +p'(x)$ where $p'(x)$ is the sum of $|x|_1$ independent random variable in $[0, 2 \delta_p]$.
We have $E[p'(x)]=|x|_1 \delta_p$ and 

\begin{eqnarray*}
& &Pr(p(x) \leq \mu(x) - \lambda)\\
& = & Pr(p'(x) \leq |x|_1 \delta_p - \lambda)\\ 
& \leq & e^{-2 \lambda^2/(4\delta_p^2|x|_1)}
 =  e^{-\lambda^2/(2\delta_p^2|x|_1)}
\end{eqnarray*}
based on Theorem~\ref{thm:addHoeffding}.

The chance constraint is met if
\begin{eqnarray*}
& & e^{-\lambda^2/(2\delta_p^2|x|_1)} \leq \alpha_p\\
& \Longleftrightarrow & -\lambda^2/(2\delta_p^2|x|_1) \leq \ln(\alpha_p)\\
& \Longleftrightarrow & \lambda^2 \geq \ln(1/\alpha_p) \cdot (2\delta_p^2|x|_1)\\
& \Longleftrightarrow & \lambda \geq \delta_p \cdot \sqrt{\ln(1/\alpha_p) \cdot 2|x|_1}
\end{eqnarray*}

Therefore, we get the following profit function based on the additive Hoeffding bound from Theorem~\ref{thm:addHoeffding}:

\begin{equation}
\hat{p}_{Hoef}(x) = \mu(x) - \delta_p \cdot \sqrt{\ln(1/\alpha_p) \cdot 2|x|_1}
\end{equation}

\subsection{Comparison of Chebyshev and Hoeffding based fitness functions}

In contrast to the fitness function $\hat{p}_{Norm}$ which is an exact reformulation, $\hat{p}_{Hoef}$ and $\hat{p}_{Cheb}$ give a conservative lower bound on the value of $P$ to be maximized.
We now consider the setting investigated for the Hoeffding bound and compare it to the use of Chebyshev's inequality. If each element is chosen independently and uniformly at random from an interval of length $2 \delta_p$ as done in Section~\ref{sec:hoeffding}, then we have $v(x) = |x|_1 \cdot \delta_p^2/3$. Based on this we can establish a condition on $\alpha_p$ which shows when $\hat{p}_{Hoef}(x) \leq \hat{p}_{Cheb}(x)$ holds.

We have 
\begin{eqnarray*}
 & \hat{p}_{Hoef}(x)   \geq  \hat{p}_{Cheb}(x)\\
 \Longleftrightarrow &  \sqrt{\ln(1/\alpha_p) \cdot 2 \cdot |x|_1} \leq  \sqrt{\frac{(1-\alpha_p)|x|_1}{3\alpha_p}} \\
 \Longleftrightarrow &  \ln(1/\alpha_p) \cdot 2 \cdot |x|_1 \leq  \frac{(1-\alpha_p)|x|_1}{3\alpha_p}\\
 \Longleftrightarrow & \ln(1/\alpha_p) \cdot  \alpha_p / (1-\alpha_p)  \leq  1/6
\end{eqnarray*}

Note that the last inequality depends only on $\alpha_p$ but not on $\delta_p$ or $|x|_1$. We will use values of $\alpha_p \in \{0.1, 0.01, 0.001\}$ in our experiments and have 
$$\ln(1/\alpha_p) \cdot  \alpha_p / (1-\alpha_p)  >  1/6$$
for $\alpha_p=0.1$ and 
$$\ln(1/\alpha_p) \cdot  \alpha_p / (1-\alpha_p)  <  1/6$$
for $\alpha_p=0.01, 0.001$. This means that the fitness function based on Chebyshev's inequality is preferable to use for $\alpha_p=0.1$ as it gives a better (tighter) value for any solution $x$ and the fitness function based on Hoeffding bounds is preferable for $\alpha_p=0.01,  0.001$. Dependent on the given instance, it might still be useful to use the less tighter fitness function as the fitness functions impose different fitness landscapes.

\section{Evolutionary algorithms}
\label{sec4}

We examine the performance of the (1+1)~EA, the (1+1)~EA with heavy-tailed mutation and a ($\mu$+1)~EA. The ($\mu$+1)~EA uses a specific crossover operator for the optimization of the chance-constrained knapsack problem with stochastic profits together with heavy-tailed mutation. All algorithms use the fitness function $f$ introduced in Section~\ref{sec3} and we will examine different choices of $\hat{p}$ in our experimental investigations.

\subsection{(1+1)~EA}

We consider a simple evolutionary algorithm called (1+1)~EA  that has been extensively studied in the area of runtime analysis. The approach is given in Algorithm~\ref{alg:oneplusoneEA}. The (1+1)~EA starts with an initial solution $x \in \{0,1\}^n$ chosen uniformly at random. It generates in each iteration a new candidate solution by standard bit mutation, i.e. by flipping each bit of the current solution with a probability of $1/n$, where $n$ is the length of the bit string. In the selection step, the algorithm accept the offspring if it is at least as good as the parent. The process is iterated until a given stopping criterion is fulfilled. While the (1+1)~EA is a very simple algorithm, it produces good results in many cases. Furthermore, it has the ability to sample new solutions globally as each bit is flipped independently with probability $1/n$. In order to overcome large inferior neighborhoods larger mutation rates might be beneficial. Allowing larger mutation rates from time to time is the idea of heavy tail mutations.

\begin{algorithm}[t]
\caption{(1+1)~EA}
\begin{algorithmic}[1]
\STATE Choose $x\in \{0,1\}^n$ to be a decision vector.
\WHILE { \textit{stopping criterion not met}}
\STATE $y\leftarrow$ flip each bit of $x$ independently with  \\ probability of $\frac{1}{n}$;
\IF{$f(y)\geq f(x)$} 
\STATE $x \leftarrow y$ ;
\ENDIF
\ENDWHILE 
\end{algorithmic}
\label{alg:oneplusoneEA}
\end{algorithm}
\subsection{Heavy tail mutations}
We also investigate the (1+1)~EA with heavy tail mutation instead of standard bit mutation.
In each operation of the heavy tail mutation operator (see Algorithm~\ref{heavyTail}, first a parameter $\theta \in [1..n/2]$ is chosen according to the discrete power law distribution $D_{n/2}^{\beta}$. 
Afterwards, each bit $n$ is flipped with probability $\theta /n$. Based on the investigations in~\cite{DBLP:conf/gecco/DoerrLMN17}, we use $\beta=1.5$ for our experimental investigations.

The heavy-tail mutation operator allows to flip significantly more bits in some mutation steps than standard bit mutation. The use of heavy tail mutations has been shown to be provably effective on the OneMax and Jump benchmark problems in theoretical investigations ~\cite{DBLP:conf/gecco/DoerrLMN17,DBLP:conf/gecco/AntipovBD20}. 
Moreover, in~\cite{DBLP:conf/gecco/XieN020} has been shown in that the use of heavy tail mutation effective improve performance of single-objective and multi-objective evolutionary algorithms for the weight chance constrained knapsack problem.
For details, on the discrete power law distribution and the heavy tail operator, we refer the reader to~\cite{DBLP:conf/gecco/DoerrLMN17}.

\begin{algorithm}[t]
\caption{The heavy-tail mutation operator}
Input: Individual $x=(x_1, \ldots,x_n) \in \{0,1\}^n$ and value $\beta$;\\
\begin{algorithmic}[1]
\STATE Choose $\theta\in [1,..,n/2]$ randomly according to $D_{n/2}^{\beta}$;
\FOR{$i=1$ to $n$}
\IF{$rand([0,1]) \leq \theta/n$} 
\STATE $y_i \leftarrow 1-x_i$;
\ELSE  \STATE $y_i \leftarrow x_i$;
\ENDIF
\ENDFOR
\RETURN $y=(y_1, \ldots,y_n)$;
\label{heavyTail}
\end{algorithmic}
\end{algorithm}

\subsection{Population-based Evolutionary Algorithm}

\begin{algorithm}[t]
\caption{$(\mu+1)$~EA}
\begin{algorithmic}[1]
\STATE Randomly generate $\mu$ initial solutions as the initial population $P$; 
\WHILE {\textit{stopping criterion not meet}}
\STATE Let $x$ and $y$ be two different individual from $P$ chosen uniformly at random;
\IF{$rand([0,1])\leq p_c$}
\STATE apply the discounted greedy uniform crossover operator to $x$ and $y$ to produce an offspring $z$.
\ELSE \STATE Choose one individual $x$ from $P$ uniformly at random and let $z$ be a copy of $x$.
\ENDIF
\STATE apply the heavy-tail mutation operator to $z$;
\IF{$f(z) \geq f(x)$}
\STATE $P \leftarrow (P \setminus \{x\}) \cup \{z\}$;
\ELSE \IF{$f(z) \geq f(y)$}
\STATE $P \leftarrow (P \setminus \{y\}) \cup \{z\}$;
\ENDIF
\ENDIF
\ENDWHILE 
\end{algorithmic}
\label{alg:muplusoneEA}
\end{algorithm}

\begin{algorithm}[t]
\caption{Discounted Greedy Uniform Crossover}
Input: Individuals $x=(x_1, \ldots, x_n)$ and $y=(y_1, \ldots, y_n)$;\\
\begin{algorithmic}[1]
\STATE Create $z = (z_1, \ldots, z_n)$ by setting $z_i \leftarrow x_i$ iff $x_i=y_i$ and $z_i \leftarrow 0$ iff $x_i!=y_i$;
\STATE Let $I = \{i \in \{1, \ldots, n\} \mid x_i!=y_i\}$;
\STATE Set $p_i' = \mu_i - u(z,i)$ for all $i \in I$; \label{discount}
\STATE Sort the items $i\in I$ in decreasing order with respect to $p_i'/w_i$ ratio;
\FOR{each $i \in I$ in sorted order}
\IF{$w(z) + w_i \leq B$}
\STATE $z_i \leftarrow 1$;
\ENDIF
\ENDFOR
\RETURN $z=(z_1, \ldots, z_n)$;
\end{algorithmic}
\label{alg:crossover}
\end{algorithm}
We also consider the population-based $(\mu+1)$-EA shown in Algorithm~\ref{alg:muplusoneEA}. 
The algorithm produces in each iteration an offspring by crossover and mutation with probability $p_c$ and by mutation only with probability $1-p_c$. We use $p_c=0.8$ for our experimental investigations.
The algorithm makes use of the specific crossover operator shown in Algorithm~\ref{alg:crossover} and heavy tail mutation. 
The crossover operator choose two different individuals $x$ and $y$ from the population $P$ and produces an offspring $z$.
All bit position where $x$ and $y$ are the same are transferred to $z$. Positions $i$ where $x_i$ and $y_i$ form the set $I$ and are different are treated in a greedy way according to the discount expected value to weight ratio. 
Setting $p_i' = \mu_i - u(z, i)$ discounts in line~\ref{discount} the expected profit by an uncertainty value based on the solution $z$ and the impact if element $i$ is added to $z$.

There are different ways of doing this. In our experiments, where the profits of the elements are chosen independently and uniformly at random, we use the calculation based on Hoeffding bounds and set
$$p_i'= \mu_i -\delta_p \cdot \left( \sqrt{\ln(1/\alpha_p) \cdot 2(|z|_1+1)} - \sqrt{\ln(1/\alpha_p) \cdot 2|z|_1} \right).$$

The expected profit $\mu_i$ is therefore discounted with the additional uncertainty that would be added according to the Hoeffding bound when adding an additional element to $z$.
Once, the discounted values $p_i'$, the elements are sorted according to $p_i'/w_i$. The final steps tries the elements of $I$ in sorted order and adds element $i \in I$ if it would not violate the weight constraint.

\section{Experimental Investigation} 
\label{sec5}
In this section, we investigate the (1+1)~EA algorithm and the (1+1)~EA with heavy-tailed mutation on several benchmarks with chance constraints and compare them to the ($\mu$ + 1)~EA algorithm with heavy-tailed mutation and new crossover operator. 

\subsection{Experimental Setup}
Our goal is to study different chance constraint settings in terms of the uncertainty level $\delta$, and the probability bound $\alpha$. We consider different well-known benchmarks from~\cite{PISINGER20052271,KelPfePis04} in their profit chance constrained versions. We consider two types of instances, uncorrelated and bounded strong correlated ones, with $n=100, 300, 500$ items. For each benchmark, we study the performance of the (1+1)~EA, (1+1)~EA with heavy-tailed mutation and ($\mu$ + 1)~EA with value of $\mu = 10$. We consider all combinations of $\alpha = 0.1, 0.01, 0.001$, and $\delta = 25, 50.0$ for the experimental investigations of the algorithms. We allow $1\,000\,000$ fitness evaluations for each of these problem parameter combinations. For each tested instance, we carry out $30$ independent runs and report the average results, standard deviation and statistical test. In order to measure the statistical validity of our results, we use the Kruskal-Wallis test with $95$\% confidence. We apply the Bonferroni post-hoc statistical correction which is used for multiple comparison of a control algorithm, to two or more algorithms~\cite{Corder09}. $X^{(+)}$ is equivalent to the statement that the algorithm in the column outperformed algorithm $X$. $X^{(-)}$ is equivalent to the statement that X outperformed the algorithm given in the column. If algorithm $X^*$ does appear, then no significant difference was determined between the algorithms.

\begin{table*}[!t]
\centering
 \caption{Experimental results for the Chebyshev based function $\hat{p}_{Cheb}$.
 }
   \label{tab:Chebyshev}
    \vspace{0.4cm}
 \begin{small}
 \scalebox{0.8}{
     \makebox[\linewidth][c]{
    \begin{tabular}{|lccc|l|r|r|c|r|r|c|r|r|c|}
    \hline
        ~ & ~ & ~ & ~ & ~ & \multicolumn{3}{|c|}{(1+1)~EA}  & \multicolumn{3}{|c|}{(1+1)~EA-HT}  & \multicolumn{3}{|c|}{($\mu$+1)~EA}   \\ 
        ~ & $B$ & $\alpha_p$ & $\delta_p$ &  & $\hat{p}_{Cheb}$ & std & stat & $\hat{p}_{Cheb}$ & std & stat & $\hat{p}_{Cheb}$ & std & stat \\ \hline
   
   
   uncorr\_100 & 2407 & 0.1 & 25 &  & 11073.5863 & 36.336192 & $2^{(*)}$, $3^{(*)}$ & 11069.0420 & 46.285605 & $1^{(*)}$, $3^{(*)}$ & 11057.4420 & 59.495722 & $1^{(*)}$, $2^{(*)}$\\ 
        ~ & ~ & ~ & 50 &  & 10863.1496 & 85.210231 & $2^{(*)}$, $3^{(*)}$ & 10889.4840 & 37.175095 & $1^{(*)}$, $3^{(*)}$ & 10883.7163 & 53.635972 & $1^{(*)}$, $2^{(*)}$ \\ 
        ~ & ~ & 0.01 & 25 &  & 10641.9089 & 63.402329 & $2^{(*)}$, $3^{(*)}$ & 10664.5974 & 29.489838 & $1^{(*)}$, $3^{(*)}$ & 10655.7251 & 43.869265 & $1^{(*)}$, $2^{(*)}$ \\ 
        ~ & ~ & ~ & 50 &  & 10054.6427 & 49.184220 & $2^{(*)}$, $3^{(*)}$ & 10066.2854 & 36.689426 & $1^{(*)}$, $3^{(*)}$ & 10064.8734 & 39.556767 & $1^{(*)}$, $2^{(*)}$ \\ 
        ~ & ~ & 0.001 & 25 &  & 9368.33053 & 46.894877 & $2^{(*)}$, $3^{(*)}$ & 9368.2483 & 34.904933 & $1^{(*)}$, $3^{(*)}$ & 9365.5257 & 40.458098 & $1^{(*)}$, $2^{(*)}$ \\ 
        ~ & ~ & ~ & 50 &  & 7475.44948 & 50.681386 & $2^{(*)}$, $3^{(*)}$ & 7490.6387 & 27.819516 & $1^{(*)}$, $3^{(*)}$ & 7497.5054 & 14.098629 & $1^{(*)}$, $2^{(*)}$ \\ \hline
        strong\_100 & 4187 & 0.1 & 25 &  & 8638.0428 & 68.740095 & $2^{(-)}$, $3^{(-)}$ & 8698.2592 & 64.435352 & $1^{(+)}$, $3^{(*)}$ & 8707.9271 & 49.633473 & $1^{(+)}$, $2^{(*)}$ \\ 
        ~ & ~ & ~ & 50 &  & 8441.9311 & 80.335771 & $2^{(-)}$, $3^{(-)}$ & 8483.1151 & 45.284814 & $1^{(+)}$, $3^{(*)}$ & 8481.0022 & 55.979520 & $1^{(+)}$, $2^{(*)}$ \\ 
        ~ & ~ & 0.01 & 25 &  & 8214.8029 & 56.705379 & $2^{(-)}$, $3^{(-)}$ & 8230.9642 & 42.084563 & $1^{(+)}$, $3^{(*)}$ & 8210.1448 & 55.148757 & $1^{(+)}$, $2^{(*)}$ \\ 
        ~ & ~ & ~ & 50 &  & 7512.3033 & 71.115520 & $2^{(-)}$, $3^{(-)}$ & 7563.5495 & 37.758812 & $1^{(+)}$, $3^{(*)}$ & 7554.7382 & 53.030592 & $1^{(+)}$, $2^{(*)}$ \\ 
        ~ & ~ & 0.001 & 25 &   & 6771.7849 & 58.314395 & $2^{(-)}$, $3^{(-)}$ & 6797.0376 & 42.944371 & $1^{(+)}$, $3^{(*)}$ & 6793.0387 & 43.492135 & $1^{(+)}$, $2^{(*)}$ \\ 
        ~ & ~ & ~ & 50 & & 4832.2084 & 88.887119 & $2^{(-)}$, $3^{(-)}$ & 4929.1483 & 52.858392 & $1^{(+)}$, $3^{(*)}$ & 4902.0006 & 44.976733 & $1^{(+)}$, $2^{(*)}$ \\ \hline \hline
        uncorr\_300 & 6853 & 0.1 & 25 &  & 34150.7224 & 167.458986 & $2^{(*)}$,$3^{(-)}$ & 34218.9806 & 164.65331 & $1^{(*)}$, $3^{(*)}$ & 34319.8500 & 177.580430 & $1^{(+)}$, $2^{(*)}$ \\ 
        ~ & ~ & ~ & 50 &  & 33749.8625 & 202.704754 & $2^{(*)}$,$3^{(-)}$ & 33827.9115 & 158.675094 & $1^{(*)}$, $3^{(*)}$  & 33992.7669 & 157.059148 & $1^{(+)}$, $2^{(*)}$  \\ 
        ~ & ~ & 0.01 & 25 & 
 & 33298.9369 & 215.463952 & $2^{(*)}$,$3^{(-)}$ & 33482.2230 & 186.361325 & $1^{(*)}$, $3^{(*)}$  & 33584.5679 & 129.781221 & $1^{(+)}$, $2^{(*)}$  \\ 
        ~ & ~ & ~ & 50 &  & 32326.5299 & 203.976688 & $2^{(*)}$,$3^{(-)}$ & 32332.5785 & 190.826414 & $1^{(*)}$, $3^{(*)}$  & 32504.2005 & 178.815508 & $1^{(+)}$, $2^{(*)}$  \\ 
        ~ & ~ & 0.001 & 25 &  & 30989.2470 & 242.861056 & $2^{(*)}$,$3^{(-)}$ & 31150.1989 & 187.329891 & $1^{(*)}$, $3^{(*)}$  & 31281.7283 & 181.280416 & $1^{(+)}$, $2^{(*)}$  \\ 
        ~ & ~ & ~ & 50 &  & 27868.2812 & 180.822780 & $2^{(*)}$,$3^{(-)}$ & 27923.1672 & 148.146917 & $1^{(*)}$, $3^{(*)}$  & 28024.3756 & 144.125407 & $1^{(+)}$, $2^{(*)}$  \\ \hline 
        strong\_300 & 13821 & 0.1 & 25 & ~ & 24795.3122 & 143.413609 & $2^{(-)}$,$3^{(*)}$& 24939.0678 & 94.941101 & $1^{(+)}$, $3^{(*)}$ & 24850.2784 & 135.783162 & $1^{(*)}$, $2^{(*)}$  \\ 
        ~ & ~ & ~ & 50 & ~ & 24525.1204 & 161.185000 & $2^{(-)}$,$3^{(*)}$ & 24585.2993 & 112.692219 & $1^{(+)}$, $3^{(*)}$ & 24589.7315 & 125.724850 & $1^{(*)}$, $2^{(*)}$\\ 
        ~ & ~ & 0.01 & 25 & ~ & 24047.9634 & 147.055910 & $2^{(-)}$,$3^{(*)}$ & 24138.6765 & 103.635233 & $1^{(+)}$, $3^{(*)}$ & 24121.8843 & 132.985469 & $1^{(*)}$, $2^{(*)}$ \\ 
        ~ & ~ & ~ & 50 & ~ & 22982.7691 & 169.377913 & $2^{(-)}$,$3^{(*)}$ & 23088.9710 & 81.229946 & $1^{(+)}$, $3^{(*)}$ & 23057.3537 & 160.481591 & $1^{(*)}$, $2^{(*)}$ \\ 
        ~ & ~ & 0.001 & 25 & ~ & 21689.9288 & 168.324844 & $2^{(-)}$,$3^{(*)}$ & 21824.5028 & 77.615607 & $1^{(+)}$, $3^{(*)}$ & 21786.4256 & 126.077269 & $1^{(*)}$, $2^{(*)}$ \\ 
        ~ & ~ & ~ & 50 & ~ & 18445.0866 & 125.747992 & $2^{(-)}$,$3^{(*)}$ & 18545.0084 & 98.512038 & $1^{(+)}$, $3^{(*)}$ & 18543.0067 & 96.526569 & $1^{(*)}$, $2^{(*)}$ \\ \hline \hline
        uncorr\_500 & 11243 & 0.1 & 25 & ~ & 58309.8801 & 266.319166 & $2^{(-)}$, $3^{(-)}$  & 58454.4069 & 295.624416 & $1^{(+)}$,$3^{(-)}$   & 58708.9818 & 157.245339 & $1^{(+)}$, $2^{(+)}$  \\ 
        ~ & ~ & ~ & 50 & ~ & 57783.7554 & 316.155254 & $2^{(-)}$, $3^{(-)}$ & 57927.2459 & 299.811063 & $1^{(+)}$, $3^{(-)}$  & 58267.9737 & 204.854052 & $1^{(+)}$, $2^{(+)}$ \\ 
        ~ & ~ & 0.01 & 25 & ~ & 57262.7885 & 330.683000 & $2^{(-)}$, $3^{(-)}$ & 57538.1166 & 260.869372 & $1^{(+)}$, $3^{(-)}$  & 57770.6524 & 178.217884 & $1^{(+)}$, $2^{(+)}$ \\ 
        ~ & ~ & ~ & 50 & ~ & 55916.4463 & 260.392742 & $2^{(-)}$, $3^{(-)}$ & 56086.6031 & 224.647105 & $1^{(+)}$, $3^{(-)}$  & 56321.8437 & 197.704397 & $1^{(+)}$, $2^{(+)}$ \\ 
        ~ & ~ & 0.001 & 25 & ~ & 54149.7603 & 364.823822 & $2^{(-)}$, $3^{(-)}$ & 54406.8517 & 249.217045 & $1^{(+)}$, $3^{(-)}$  & 54806.6815 & 170.082092  & $1^{(+)}$, $2^{(+)}$ \\ 
        ~ & ~ & ~ & 50 & ~ & 50124.9811 & 265.408552 & $2^{(-)}$, $3^{(-)}$ & 50312.3993 & 286.632525 & $1^{(+)}$, $3^{(-)}$  & 50672.0950 & 197.712768 & $1^{(+)}$, $2^{(+)}$ \\ \hline 
        strong\_500 & 22223 & 0.1 & 25 & ~ & 41104.1611 & 321.324820 & $2^{(*)}$, $3^{(*)}$ & 41523.8952 & 222.691441 & $1^{(*)}$, $3^{(*)}$ & 41458.8477 & 238.463764 & $1^{(*)}$, $2^{(*)}$ \\ 
        ~ & ~ & ~ & 50 & ~ & 40834.8213 & 243.308935 & $2^{(*)}$, $3^{(*)}$ & 41067.8559 & 229.706142 & $1^{(*)}$, $3^{(*)}$ & 41043.6296 & 173.586544 & $1^{(*)}$, $2^{(*)}$ \\ 
        
        ~ & ~ & 0.01 & 25 & ~ & 40248.7094 & 289.114488 & $2^{(*)}$, $3^{(*)}$ & 40567.8724 & 133.387473 & $1^{(*)}$, $3^{(*)}$ & 40448.5671 & 206.754226 & $1^{(*)}$, $2^{(*)}$ \\ 
        
        ~ & ~ & ~ & 50 & ~ & 38831.0336 & 298.888606 & $2^{(*)}$, $3^{(*)}$ & 39123.3879 & 120.110352 & $1^{(*)}$, $3^{(*)}$ & 38984.3118 & 169.701352 & $1^{(*)}$, $2^{(*)}$ \\ 
        ~ & ~ & 0.001 & 25 & ~ & 37201.8768 & 273.119842  & $2^{(*)}$, $3^{(*)}$ & 37490.7767 & 118.382846 & $1^{(*)}$, $3^{(*)}$ & 37395.7375 & 164.601365 & $1^{(*)}$, $2^{(*)}$ \\ 
        ~ & ~ & ~ & 50 & ~ & 32880.2003 & 272.672330 & $2^{(*)}$, $3^{(*)}$ & 33013.4535 & 172.524052 & $1^{(*)}$, $3^{(*)}$ & 32951.6884 & 206.900731 & $1^{(*)}$, $2^{(*)}$ \\ \hline 
    \end{tabular}
     }}
      \end{small}
\end{table*}
\section{Experimental Results} 
We consider now the results for the (1+1)~EA, (1+1)~EA with heavy-tailed mutation and 
($\mu$ + 1)~EA with heavy-tailed mutation and specific crossover algorithm based on Chebyshev's inequality and Hoeffding bounds for the benchmark set.

We first consider the optimization result obtained by the above mentioned algorithms using Chebyshev's inequality for the combinations of $\alpha_p$ and $\delta_p$. The experimental results are shown in Table~\ref{tab:Chebyshev}.
The results show that (1 + 1)~EA with heavy-tailed mutation is able to achieve higher average results for the instances with $100, 300, 500$ items for type bounded strongly correlated in most of the cases for all $\alpha_p$ and $\delta_p$ combinations.
It can be observed that for the instance with $100$ uncorrelated items the (1 + 1)~EA with heavy-tailed mutation outperforms all algorithms for $\alpha = 0.1$ and $\delta = 50$ and for $\alpha = 0.01$ $\delta = 25, 50$, respectively. However, $(\mu + 1)~EA$ can improve on the optimization result for small $\alpha_p$ and high $\delta_p$ values, i.e. $\alpha_p$ = $0.001$, $\delta_p$ = $50$. 

It can be observed that ($\mu$ + 1)~EA obtains the highest mean value for the instance with $300$ and $500$ items for the uncorrelated type. Furthermore, the statistical tests show that for all combinations of $\alpha_p$ and $\delta_p$ ($\mu$ + 1)~EA significantly outperforms (1 + 1)~EA and (1 + 1)~EA with heavy-tailed mutation. For example, for the instance with $300, 500$ items uncorrelated and for $100$ items bounded strongly correlated the statistical tests show that ($\mu$ + 1)~EA and (1 + 1)~EA with heavy-tailed mutation outperforms (1 + 1)~EA. For the other settings there is no statistical significant difference in terms of the results between all algorithms.  

Table~\ref{tab:Hoeffding} shows the results obtained by the above mentioned algorithms using Hoeffding bounds for the combinations of $\alpha_p$ and $\delta_p$ and statistical tests. 
The results show that (1+1)~EA with heavy-tailed mutation obtains the highest mean values compared to the results obtained by (1+1)~EA and ($\mu$ + 1)~EA for each setting for the instance with $100$ items for both types, uncorrelated and bounded strongly correlated. 
Similar to the previous investigation in the case for the instances with $300$ items, the (1+1)~EA with heavy-tailed mutation obtains the highest mean values compared to the results obtained by other algorithms in most of the cases. However, the solutions obtained by ($\mu+1$)~EA has significantly better performance than in the case for $\alpha_p$ = $0.1, 0.001$, $\delta_p$ = $25$.

The use of the heavy-tailed mutation when compared to the use of standard bit mutation in (1+1)~EA achieves a better performance for all cases.
Furthermore, the statistical tests show that for most combinations of $\alpha_p$ and $\delta_p$, the (1+1)~EA with heavy-tailed mutation significantly outperforms the other algorithms. Overall, it can be observed that using a heavy-tailed mutation both algorithms achieve higher average results. This can be due to the fact that a higher number of bits can be flipped than in the case of standard bit mutations flipping every bit with probability $1/n$.
\begin{table*}[!t]
    \centering
 \caption{Experimental results for the Hoeffding based function $\hat{p}_{Hoef}$.
 }
 \vspace{0.4cm}
 \begin{small}
 \scalebox{0.8}{
     \makebox[\linewidth][c]{
     \begin{tabular}{|lccc|l|r|r|c|r|r|c|r|r|c|}
    \hline
        ~ & ~ & ~ & ~ & ~ & \multicolumn{3}{|c|}{(1+1)~EA}  & \multicolumn{3}{|c|}{(1+1)~EA-HT}  & \multicolumn{3}{|c|}{($\mu$+1)~EA}   \\ 
         & $B$ & $\alpha_p$ & $\delta_p$ & ~ &  $\hat{p}_{Hoef}$ & std & stat & $\hat{p}_{Hoef}$ & std & stat & $\hat{p}_{Hoef}$ & std & stat \\ \hline
      
        uncorr\_100 & 2407 & 0.1 & 25 & & 10948.7292 & 90.633230 & $2^{(-)}$, $3^{(*)}$  & 11016.8190 & 49.768932 & $1^{(+)}$,$3^{(+)}$  & 10981.3880 & 37.569308 & $1^{(*)}$, $2^{(-)}$ \\ 
        ~ & ~ & ~ & 50 &  & 10707.1094 & 43.869094 & $2^{(-)}$, $3^{(*)}$ & 10793.1175 & 58.150646 & $1^{(+)}$,$3^{(+)}$ & 10708.6094 & 44.384035 & $1^{(*)}$,$2^{(-)}$ \\ 
        ~ & ~ & 0.01 & 25 & ~ & 10836.0906 & 91.332983 & $2^{(-)}$, $3^{(*)}$  & 10928.3054 & 45.464936 & $1^{(+)}$,$3^{(+)}$ & 10866.9831 & 45.408500 & $1^{(*)}$,$2^{(-)}$ \\ 
        ~ & ~ & ~ & 50 & ~ & 10482.6216 & 46.444510 & $2^{(-)}$, $3^{(*)}$  & 10611.1895 & 69.341044 & $1^{(+)}$,$3^{(+)}$ & 10477.2328 & 47.065426 & $1^{(*)}$,$2^{(-)}$ \\ 
        ~ & ~ & 0.001 & 25 & ~ & 10765.3289 & 68.565293 & $2^{(-)}$, $3^{(*)}$  & 10862.7124 & 49.091526 & $1^{(+)}$,$3^{(+)}$ & 10784.7286 & 38.187390 &  $1^{(*)}$,$2^{(-)}$\\ 
        ~ & ~ & ~ & 50 & ~ & 10263.9426 & 90.504901 & $2^{(-)}$, $3^{(*)}$ & 10487.5621 & 32.625499 & $1^{(+)}$,$3^{(+)}$ & 10309.8572 & 44.811326 & $1^{(*)}$,$2^{(-)}$\\ \hline
        strong\_100 & 4187 & 0.1 & 25 &  ~ & 8553.1744 & 74.046187 & $2^{(-)}$, $3^{(*)}$ & 8640.05156 & 39.413105 & $1^{(+)}$,$3^{(+)}$ & 8588.4894 & 53.878268 & $1^{(*)}$, $2^{(-)}$ \\ 
        ~ & ~ & ~ & 50 & ~ & 8264.8129 & 63.309264 & $2^{(-)}$, $3^{(*)}$ & 8398.4354 & 46.013234 &  $1^{(+)}$,$3^{(+)}$ & 8273.9670 & 41.403505 & $1^{(*)}$, $2^{(-)}$ \\ 
        ~ & ~ & 0.01 & 25 & ~ & 8422.9258 & 70.464985 & $2^{(-)}$, $3^{(*)}$& 8540.2095 & 63.072560 &  $1^{(+)}$,$3^{(+)}$ & 8447.8489 & 59.841707 & $1^{(*)}$, $2^{(-)}$ \\ 
        ~ & ~ & ~ & 50 & ~ & 7996.0193 & 65.822419 & $2^{(-)}$, $3^{(*)}$ & 8181.2980 & 45.667034 &  $1^{(+)}$,$3^{(+)}$ & 8013.1724 & 56.445427 & $1^{(*)}$, $2^{(-)}$ \\ 
        ~ & ~ & 0.001 & 25 & ~ & 8338.5159 & 57.880350 & $2^{(-)}$, $3^{(*)}$ & 8460.7513 & 53.402755 &  $1^{(+)}$,$3^{(+)}$ & 8362.9405 & 51.607219 & $1^{(*)}$, $2^{(-)}$ \\ 
        ~ & ~ & ~ & 50 & ~ & 7794.1245 & 80.411946 & $2^{(-)}$, $3^{(*)}$ & 8017.8843 & 53.266120 &  $1^{(+)}$,$3^{(+)}$ & 7833.5575 & 37.293481 & $1^{(*)}$, $2^{(-)}$ \\ \hline \hline
        uncorr\_300 & 6853 & 0.1 & 25 & ~ & 33831.9693 & 181.485453 & $2^{(-)}$, $3^{(-)}$ & 34118.7631 & 200.095911 & $1^{(+)}$, $3^{(*)}$ & 34129.8891 & 172.788856 & $1^{(+)}$, $2^{(*)}$\\ 
        ~ & ~ & ~ & 50 & ~ & 33380.4952 & 157.014552 & $2^{(-)}$, $3^{(-)}$ & 33715.2964 & 199.074378 & $1^{(+)}$, $3^{(*)}$ & 33662.2668 & 124.206823 & $1^{(+)}$, $2^{(*)}$ \\ 
        ~ & ~ & 0.01 & 25 & ~ & 33655.5737 & 234.136500 & $2^{(-)}$, $3^{(-)}$ & 34014.3456 & 200.488072 & $1^{(+)}$, $3^{(*)}$ & 33962.8643 & 161.560953 & $1^{(+)}$, $2^{(*)}$ \\ 
        ~ & ~ & ~ & 50 & ~ & 32933.5174 & 291.623690 & $2^{(-)}$, $3^{(-)}$ & 33327.8984 & 235.915481 & $1^{(+)}$, $3^{(*)}$ & 33277.4015 & 142.387738 & $1^{(+)}$, $2^{(*)}$ \\ 
        ~ & ~ & 0.001 & 25 & ~ & 33515.7445 & 219.707660 & $2^{(-)}$, $3^{(-)}$ & 33806.1572 & 184.532069 & $1^{(+)}$, $3^{(*)}$ & 33835.4528 & 149.327823 & $1^{(+)}$, $2^{(*)}$ \\ 
        ~ & ~ & ~ & 50 & ~ & 32706.4466 & 176.599463 & $2^{(-)}$, $3^{(-)}$ & 33112.7494 & 177.218747 & $1^{(+)}$, $3^{(*)}$ & 32940.4397 & 173.836538 & $1^{(+)}$, $2^{(*)}$ \\ \hline
        strong\_300 & 13821 & 0.1 & 25 & ~ & 24602.1254 & 171.596469 & $2^{(-)}$, $3^{(-)}$ & 24848.3209 & 100.078545 & $1^{(+)}$, $3^{(+)}$  & 24734.7210 & 127.268428 & $1^{(+)}$,$2^{(-)}$ \\ 
        ~ & ~ & ~ & 50 & ~ & 24184.8938 & 125.755762 & $2^{(-)}$, $3^{(-)}$ & 24457.7279 & 118.679623 & $1^{(+)}$, $3^{(+)}$ & 24205.9660 & 116.049342 & $1^{(+)}$,$2^{(-)}$ \\ 
        ~ & ~ & 0.01 & 25 & ~ & 24476.1412 & 159.274566 & $2^{(-)}$, $3^{(-)}$ & 24638.0751 & 105.088783 & $1^{(+)}$, $3^{(+)}$ & 24538.4199 & 101.959196 & $1^{(+)}$,$2^{(-)}$ \\ 
        ~ & ~ & ~ & 50 & ~ & 23653.3561 & 225.087307 & $2^{(-)}$, $3^{(-)}$ & 24060.0806 & 87.242862 & $1^{(+)}$, $3^{(+)}$ & 23830.8655 & 85.829604 & $1^{(+)}$,$2^{(-)}$ \\ 
        ~ & ~ & 0.001 & 25 & ~ & 24256.4468 & 173.293324 & $2^{(-)}$, $3^{(-)}$ & 24558.9506 & 105.253206 & $1^{(+)}$, $3^{(+)}$ & 24345.4340 & 144.094192 & $1^{(+)}$,$2^{(-)}$ \\ 
        ~ & ~ & ~ & 50 & ~ & 23377.6774 & 143.350899 & $2^{(-)}$, $3^{(-)}$ & 23843.7258 & 114.231223 & $1^{(+)}$, $3^{(+)}$ & 23520.1166 & 112.403711 & $1^{(+)}$,$2^{(-)}$ \\ \hline \hline
        uncorr\_500 & 11243 & 0.1 & 25 & ~ & 57995.2668 & 285.959899 & $2^{(-)}$, $3^{(-)}$  & 58286.1443 & 253.880622 & $1^{(+)}$, $3^{(-)}$ & 58527.7062&179.624520 & $1^{(+)}$,$2^{(+)}$ \\
        ~ & ~ & ~ & 50 & ~ & 57331.7069 & 319.089163 & $2^{(-)}$, $3^{(-)}$  & 57825.9426 & 227.649351 & $1^{(+)}$, $3^{(-)}$ & 57899.9614 & 167.585846 & $1^{(+)}$,$2^{(+)}$ \\ 
        
        ~ & ~ & 0.01 & 25 & ~ & 57757.1719 & 290.254639 & $2^{(-)}$, $3^{(-)}$  & 58023.1930 & 277.702516 & $1^{(+)}$, $3^{(-)}$ & 58224.2474 & 211.715398 & $1^{(+)}$,$2^{(+)}$ \\ 
        
        ~ & ~ & ~ & 50 & ~ & 56787.0897 & 411.706381 & $2^{(-)}$, $3^{(-)}$ & 57367.9869 & 206.916491 & $1^{(+)}$, $3^{(+)}$ & 57309.9927 & 227.397029 & ~$1^{(+)}$,$2^{(-)}$ \\ 
        
        ~ & ~ & 0.001 & 25 & ~ & 57519.6613 & 379.930530 & $2^{(-)}$, $3^{(-)}$  & 57910.4812 & 250.540248 & $1^{(+)}$, $3^{(-)}$ & 58052.4481 & 182.866780 & $1^{(+)}$,$2^{(+)}$ \\ 
        
        ~ & ~ & ~ & 50 & ~ & 56446.5408 & 273.663433 & $2^{(-)}$, $3^{(-)}$  & 57018.3566 & 253.684943 & $1^{(+)}$, $3^{(+)}$ & 56942.4016 & 183.464200 & $1^{(+)}$,$2^{(-)}$ \\ \hline

        strong\_500 & 22223 & 0.1 & 25 & ~ & 41060.1634 & 306.686391 & $2^{(-)}$, $3^{(*)}$ & 41397.7895 & 146.844521 & $1^{(+)}$, $3^{(+)}$ & 41186.7266& 213.577571 & $1^{(*)}$, $2^{(-)}$ \\ 
        
        ~ & ~ & ~ & 50 & ~ & 40244.7545 & 272.646652 & $2^{(-)}$, $3^{(*)}$ & 40897.9183 & 231.639926 & $1^{(+)}$, $3^{(+)}$ & 40543.1279 & 221.615657 & $1^{(*)}$, $2^{(-)}$  \\ 
        
        ~ & ~ & 0.01 & 25 & ~ & 40800.7084 & 271.459688 & $2^{(-)}$, $3^{(*)}$ & 41204.2676 & 179.999423 & ~$1^{(+)}$, $3^{(+)}$& 40967.2373 & 229.232904 & $1^{(*)}$, $2^{(-)}$ \\ 
        
        ~ & ~ & ~ & 50 & ~ & 39839.2235 & 271.298804 & $2^{(-)}$, $3^{(*)}$ & 40445.3621 & 157.093438 & $1^{(+)}$, $3^{(+)}$ & 40012.5861 & 189.516720 & $1^{(*)}$, $2^{(-)}$  \\ 
        
        ~ & ~ & 0.001 & 25 & ~ & 40561.9235 & 348.722449 & $2^{(-)}$, $3^{(*)}$ & 41038.9246 & 136.670185 & $1^{(+)}$, $3^{(+)}$ & 40768.4056 & 206.509572 & $1^{(*)}$, $2^{(-)}$  \\ 
        
        ~ & ~ & ~ & 50 & ~ & 39404.7836 & 249.449911 & $2^{(-)}$, $3^{(*)}$ & 40087.0447 & 167.453651 & $1^{(+)}$, $3^{(+)}$ & 39561.6572 & 216.629134 & , $1^{(*)}$, $2^{(-)}$  \\ \hline
            \end{tabular}
             }}
             \label{tab:Hoeffding}
              \end{small}
\end{table*}

\section{Conclusions}
Stochastic problems play an important role in many real-world applications. Based on real-world problems where profits in uncertain environments should be guaranteed with a good probability, we introduced the knapsack problem with chance constrained profits. We presented fitness functions for different stochastic settings that allow to maximize the profit value $P$ such that the probability of obtaining a profit less than $P$ is upper bounded by $\alpha_p$. In our experimental study, we examined different types of evolutionary algorithms and compared their performance on stochastic settings for classical knapsack benchmarks.

\section*{Acknowledgements}
We thank Simon Ratcliffe, Will Reid and Michael Stimson for very useful discussions on the topic of this paper. This work has been supported by the Australian Research Council (ARC) through grant FT200100536, and by the South Australian Government through the Research Consortium "Unlocking Complex Resources through Lean Processing".

\bibliographystyle{unsrt}
\bibliography{Bib.bib}

\end{document}